\documentclass{llncs}
\usepackage{xspace}
\usepackage[usenames,dvipsnames]{color}
\usepackage{hyperref}
\usepackage{longtable}
\usepackage{graphicx}
\usepackage{subcaption}
\usepackage[all]{xy}
\usepackage{amsfonts}
\usepackage{tabularx}

\newcommand{\papertitle}{$\#\exists$SAT: Projected Model Counting}
\newcommand{\authornames}{Rehan Abdul Aziz \and Geoffrey Chu \and Christian Muise \and Peter Stuckey}

\newcommand{\printsolver}[1]{\textsc{#1}\xspace}

\newcommand{\clasp}{\printsolver{clasp}}

\newcommand{\dsharp}{\printsolver{dSharp}}
\newcommand{\dsharpp}{\printsolver{dSharp\_P}}
\newcommand{\sharpsat}{\printsolver{sharpSAT}}
\newcommand{\sharpclasp}{\printsolver{\#clasp}}
\newcommand{\dtoc}{\printsolver{d2c}}
\newcommand{\minisat}{\printsolver{Minisat}}

\newcommand{\true}{\mathit{true}}
\newcommand{\false}{\mathit{false}}
\newcommand{\var}{\mathit{var}}
\newcommand{\vars}{\mathit{vars}}
\newcommand{\countt}{\mathit{ct}}
\newcommand{\size}{\mathit{size}}
\newcommand{\add}{\mathit{add}}
\newcommand{\bl}{\mathit{bl}}
\newcommand{\dec}{\mathit{dec}}
\newcommand{\litAtNode}{\mathit{litAtNode}}
\newcommand{\litWithHash}{\mathit{litWithHash}}
\newcommand{\Nodes}{\mathit{Nodes}}
\newcommand{\op}{\mathit{op}}

\newcommand{\VV}{{\cal V}}
\newcommand{\PP}{{\cal P}}
\newcommand{\NN}{{\cal N}}

\newcommand{\priority}{priority\xspace}
\newcommand{\npriority}{non-priority\xspace}

\newcommand{\solCubeRatio}{\log_2(\frac{\mathit{\#sols}}{\mathit{\#cubes}})}

\newcommand{\residual}[2]{#1|_{#2}}

\newcommand{\bif}{\textbf{if}\xspace}
\newcommand{\belse}{\textbf{else}\xspace}
\newcommand{\belseif}{\textbf{elif}\xspace}
\newcommand{\bfor}{\textbf{for}\xspace}

\newcommand{\breturn}{\textbf{return }\xspace}

\newcommand{\bbreak}{\textbf{break}\xspace}

\begin{document}

\title{\papertitle}
\author{\authornames}
\institute{National ICT Australia, Victoria Laboratory \thanks{NICTA is funded by the Australian Government as represented by the
Department of Broadband, Communications and the Digital Economy and the
Australian Research Council through the ICT Centre of Excellence program.} \\ 
Department of Computing and Information Systems \\ The University of Melbourne\\}

\maketitle

\begin{abstract}
Model counting is the task of computing the number of assignments to
variables $\VV$ that satisfy a given propositional theory $F$.
The model counting problem is denoted as \#SAT. 
Model counting is an essential tool in probabilistic reasoning.
In this paper, we introduce the problem of model 
counting projected on a 
subset of original variables that we call \emph{\priority}
variables $\PP \subseteq \VV$. 
The task is to compute the number of assignments to $\PP$ 
such that there exists an extension to \emph{\npriority} variables $\VV
\setminus \PP$ that satisfies $F$. We denote this as $\#\exists$SAT.
Projected model counting arises when some parts of the model are irrelevant
to the counts, in particular when we require additional variables to model
the problem we are counting in SAT.
We discuss three different approaches to $\#\exists$SAT (two of which are novel),
and compare their performance on different benchmark problems.
\end{abstract}

\section{Introduction}

Model counting is the task of computing the number of models of a given
propositional theory, represented as a set of clauses (SAT).
Often, instead of the original model count, we are interested in model count
projected on a set of variables $\PP$.

Given a problem on variables $\PP$, we may need to introduce
additional variables to encode the constraints on the variables $\PP$ into
Boolean clauses in the propositional theory $F$. Counting the models of
$F$ does not give the correct count if the new variables are not
\emph{functionally defined} by the original variables $\PP$. 
Thankfully, most methods of encoding constraints introduce new variables that
are functionally defined by original variables, but there are cases
where the most efficient encoding of constraints does not enjoy this
property.
%
%
Hence we should consider projected model counting for these kinds
of problems.

Alternatively, in the counting problem itself, we may only be interested in some of
the variables involved in the problem.  Unless the interesting variables
functionally define the uninteresting variables, we need projected model
counting.
An example is in \emph{evaluating} robustness of a given solution. The goal is to count the changes that can be made to a subset of variables in the solution such that it still remains a solution (possibly after allowing some repairs, e.g. in supermodels of a propositional theory \cite{supermodels}). The variables representing change are \priority variables.
In our benchmarks, we consider an example from the planning domain, 
where we are interested in robustness of a given partially ordered plan to the initial conditions, i.e.,
we want to count the number of initial states, such that the given partially ordered plan still reaches
the given goal state(s).

Projected model counting is a challenging problem that has received little
attention. 
It is in \#P\textsuperscript{NP}. If all the variables are priority variables, then it becomes
a \#SAT problem (\#P), and if all variables are \npriority variables, then it 
reduces to SAT (NP).
There has been little development of specialized algorithms 
for projected model counting in the literature. 
Some dedicated attempts at solving the problem are presented in
\cite{sharpcdcl} and 
\cite{clasp_projection}. In the latter, the primary motivation is solution
enumeration, and not counting. 
Closely related problems are
projection or \emph{forgetting} in formulas 
that are in deterministic decomposable negation normal form 
(d-DNNF~\cite{ddnnf}) \cite{kc_map}, and \emph{Boolean quantifier elimination}
\cite{existential_quantification,quantifier_elimination_auto,dds}. 

In this paper, we present three different approaches for projected model counting. 
\begin{itemize}
\item
The first technique is straight-forward and its basic idea is to modify DPLL-based model counters
to search first on the \priority variables, followed by finding only a single solution for
the remaining problem.
This technique is not novel and has been proposed in \cite{sharpcdcl}. It has also been suggested in \cite{planning_projection} in a slightly
different context.
Unlike \cite{sharpcdcl} which uses external calls to \minisat to check satisfiability
of \npriority components, we handle all computations within the solver.
\item
The second approach is a significant 
extension of the algorithm presented in \cite{clasp_projection}.
The basic idea is that every time a solution $S$ is found, we generalize it by greedily finding a subset of literals $S'$
that are sufficient to satisfy all clauses of the problem. By adding $\neg S'$ as a clause, we save an exponential amount of 
search that would visit all 
extensions of $S'$. This extension conveniently blends in the original algorithm of \cite{clasp_projection}, 
which has the property that the number of blocking clauses are polynomial in the number of \priority variables
at any time during the search.
\item 
Our third technique is a novel idea 
which reuses model counting algorithms:
computing the d-DNNF of the original problem, 
forgetting the \npriority variables in the d-DNNF,
converting the resulting DNNF to CNF, and counting the models of this CNF.
\end{itemize}
We compare these three techniques on different benchmarks
to illustrate their strengths and weaknesses.

\section{Preliminaries}
\label{sec:preliminaries}

We consider a finite set $\VV$ of propositional variables.
A literal $l$ is a variable $v \in \VV$ or its negation $\neg v$.
The negation of a literal $\neg l$ is $\neg v$ if $l = v$ or
$v$ if $l = \neg v$. 
Let $\var(l)$ represent the variable of the literal, i.e.,
$\var(v) = \var(\neg v) = v$.
A clause is a set of literals that represents their disjunction,
we shall write in parentheses $(l_1, \ldots, l_n)$.
For any formula (e.g. a clause) $C$, let $\vars(C)$ be the set of variables appearing in $C$.
A formula $F$ in conjunctive normal form (CNF) is a conjunction of clauses,
and we represent it simply as a set of clauses.
An assignment $\theta$ is a set of literals, such that if $l \in \theta$,
then $\neg l \notin \theta$. We shall write them using set notation.
Given an assignment $\theta$ then $\neg \theta$ is the clause 
$\bigvee_{l \in \theta} \neg l$.
Given an assignment $\theta$ over $\VV$ 
and set of variables $P$
then $\theta_P = \{ l ~|~ l \in
\theta, \var(l) \in P \}$

Given an assignment $\theta$, the \emph{residual} of a CNF $F$ w.r.t. $\theta$ is
written $\residual{F}{\theta}$ and is obtained by removing each clause $C$
in $F$ such that there exists a literal $l \in C \cap \theta$, and simplifying
the remaining clauses by removing all literals from them whose negation is in $\theta$.
We say that an assignment $\theta$ is a solution \emph{cube}, or simply a cube, 
of $F$ iff $\residual{F}{\theta}$
is empty. The size of a cube $\theta$, $\size(\theta)$ is equal to $2^{|\VV| - |\theta|}$.
A solution in the classical sense is a cube of size 1. The model count of $F$,
written, $\countt(F)$ is the number of solutions of $F$.

We 
consider a set of \priority variables $\PP \subseteq \VV$. Let the \npriority
variables be $\NN$, i.e., $\NN = \VV \setminus \PP$. 
Given a cube $\theta'$ of formula $F$, then
$\theta \equiv \theta'_\PP$ is a \emph{projected cube} of $F$.
The size of the projected cube is equal to $2^{|\PP| - |\theta|}$.
The projected
model count of $F$, $\countt(F, \PP)$ is equal to the number of projected cubes
of size 1. 
The projected model count can also be defined as the number of assignments
$\theta$ s.t. $\vars(\theta) = \PP$ and there exists an assignment
$\theta'$ s.t. $\vars(\theta') = \NN$ and $\theta \cup \theta'$
is a solution of $F$.

A Boolean formula is in \emph{negation normal form} (NNF) 
iff the only sub-formulas
that have negation applied to them are propositional variables.
An NNF formula is \emph{decomposable}
(DNNF) iff for all conjunctive formulae $c_1 \wedge \cdots \wedge c_n$ 
in the formula, the sets of variables of conjuncts are pairwise disjoint,
$vars(c_i) \cap var(c_j) = \emptyset, 1 \leq i \neq j \leq n$.
Finally,
a DNNF is \emph{deterministic} (d-DNNF) if for all disjunctive formulae 
$d_1 \vee \cdots \vee d_n$ in the formula,
the disjuncts are pairwise logically inconsistent,
$d_i \wedge d_j$ is unsatisfiable, $1 \leq i \neq j \leq n$.
A d-DNNF is typically represented as a tree or DAG with inner nodes and leaves being OR/AND operators
and literals respectively. 
Model counting on d-DNNF can be performed in polynomial
time (in d-DNNF size) by first computing the satisfaction
probability and then multiplying the satisfaction probability with
total number of assignments. Satisfaction probability can be computed by
evaluating the arithmetic expression that we get by replacing each literal with 0.5, $\vee$ with $+$
and $\wedge$ with $\times$ in the d-DNNF.


\section{Model Counting}

In this section we review two algorithms for model counting that are necessary for understanding
the remainder of this paper. For a more complete treatment 
of model counting algorithms,
see \cite{model_counting}.

\subsection{Solution enumeration using SAT solvers}


In traditional DPLL-algorithm~\cite{dpll}, 
once a decision literal is retracted, it is guaranteed
that all search space extending the current assignment has been exhausted. Due to this,
we can be certain that the search procedure is complete and does not miss any solution.
This is not true, however, for modern SAT solvers~\cite{moskewicz01} 
that use random restarts and First-UIP backjumping.
In the latter, the search backtracks to the last point
in search where the learned clause is asserting, and that might mean backjumping over valid solution space.
It is not trivial to infer from the current state of the solver which solutions have already
been seen and therefore, 
to prevent the search from finding an already visited solution $\theta$,
SAT solvers add the blocking clause $\neg \theta$ 
in the problem formulation as soon
as $\theta$ is found.

\subsection{DPLL-style model counting}


One of the most successful approaches for model counting 
extends the DPLL algorithm (see \cite{sumprod,cachet,sharpsat}). 
Such model counters borrow many useful features
from SAT solvers such as nogood learning, watched literals and backjumping etc to prune parts
of search that have no solution. However, they have three additional important optimizations
that make them more efficient at model counting as compared to solution enumeration using
a SAT solver. A key property of all these optimizations is that their implementation
relies on actively maintaining the residual formula during the search. This requires visiting
all clauses in the worst case at every node in the search tree. 

Say we are solving $F$ and the current assignment is $\theta$.
The first optimization in model counting is cube detection; as soon as the residual is empty, we can
stop the search and increment our model count by $\size(\theta)$. This avoids continuing the search
to visit all extensions of the cube since all of them are solutions of $F$. The second optimization
is \emph{caching} \cite{dpll_with_caching} which reuses model counts of previously encountered sub-problems instead of solving
them again as follows. Say we have computed the model count below $\theta$ and it is equal to $c$, we store $c$
against $\residual{F}{\theta}$. If, later in the search, our assignment is $\theta'$
and $\residual{F}{\theta} = \residual{F}{\theta'}$, then we can simply increment our count by $c$
by looking up the residual. The third optimization is \emph{dynamic decomposition}
and it relies on the following property of Boolean formulas: given a formula $G$, if (clauses of) $G$ can be
split into $G_1, \ldots, G_n$ such that 
$\vars(G_i) \cap \vars(G_j) = \emptyset, 1 \leq i \neq j \leq n$,
and $\bigcup_{i \in 1..n} \vars(G_i) = \vars(G)$,
then $\countt(G) = \countt(G_1) \times \ldots \times \countt(G_n)$. Model counters use this property
and split the residual into disjoint components and 
count the models of each component and multiply them to get the count of the residual. 
Furthermore, when used with caching,
the count of each component is stored against it so that if a component appears again in the search, 
then we can retrieve its count instead of computing it again.

\section{Projected Model Counting}

In this section, we present three techniques for projected model counting.

\subsection{Restricting search to \priority variables}
\label{subsec:dsharpp}

This algorithm works by slightly modifying the DPLL-based model counters as follows.
First, when solving any component, we only allow search decisions on \npriority variables if the component does not
have any \priority variables. Second, if we find a cube for a component, then the size of that cube
is equal to 2 to the power of number of \priority variables in the component. 
Finally, as soon as we find a cube for a component, we recursively mark all its \emph{parent} components
(components from earlier decision levels whose decomposition yielded the current component),
that do not have any \priority variables as solved. As a result, the count of 1 from the last
component is propagated to all parent components whose clauses are exclusively on \npriority variables.
Essentially, we store the fact that such components are \emph{satisfiable}.

\begin{example}
\label{ex:main}
Consider the 
formula $F$ with \priority variables $\PP = \{p,q,r\}$ and \npriority variables $\NN = \{x,y,z\}$.
$$
(\neg q, x, \neg p), (\neg r, \neg y, z), (r, \neg z, \neg p), (z, y, \neg p, r), (r, z, \neg y, \neg p), (p,q)
$$
Here is the trace of a possible execution using the algorithm in this
subsection. We represent a component as a pair of (unfixed) variables and residual clauses.
\\
1a. Decision $p$. The problem splits into 
$C_1 = (\{q,x\}, \{(\neg q,x)\})$ and \\
$C_2 = (\{r,y,z\}, \{(\neg r, \neg y, z), (r, \neg z), (z, y, r), (r, z, \neg y)\})$. \\
2a. We solve $C_1$ first. Decision $\neg q$. We get $C_3 = (\{x\}, \emptyset)$
and $\countt(C_3, \PP) = 1$ (trivial), we backtrack to $C_1$.
2b. Decision $q$, propagates $x$, and it is a solution. We backtrack
and set $\countt(C_1, \PP) = \countt(C_3, \PP) + 1 = 2$. \\
2c. 
Now, we solve $C_2$. Decision $r$ gives $C_4 = (\{y,z\}, \{(\neg y,z)\})$. \\
3a. Decision $z$, we get $C_5 = (\{y\}, \emptyset)$ and $\countt(C_5, \PP) = 1$.
We backtrack to level $C_2$ setting $\countt(C_4, \PP) = 1$ since the last decision
was a non-priority variable. \\
2d. Decision $\neg r$ fails 
(propagates $z$, $y$, $\neg y$). We set $\countt(C_2, \PP) = \countt(C_4, \PP) = 1$
and backtrack to root $F$ to try the other branch. \\
1b. Decision $\neg p$, propagates $q$ and gives $C_6 = (\{x\}, \emptyset)$
and $C_7 = (\{r,y,z\}, \{(\neg r,\neg y,z)\})$. 
We note that 
$\countt(C_6, \PP) = 1$ (trivial)
and move on to solve $C_7$. \\
2e. Decision $\neg r$ gives $C_8 = (\{y\}, \emptyset)$ and $C_9 = (\{z\}, \emptyset)$
with counts 1 each. We go back to $C_7$ to try the other branch. \\
2f. Decision $r$ gives $C_{10} = (\{y,z\}, \{(\neg y,z)\})$ which is the same
as $C_4$ which has the count of $1$. 
Therefore, $\countt(C_7, \PP) = \countt(C_8, \PP) \times \countt(C_9, \PP) + \countt(C_4, \PP) = 2$.
All components are solved, and there are no more choices to be tried, we go back to root to
get the final model count. 

A visualization of the search is shown in Figure~\ref{fig:model}.
The overall count is $\countt(F, \PP) = \countt(C_1, \PP) \times \countt(C_2, \PP) +
\countt(C_6, \PP) \times \countt(C_7, \PP) = 4$. \qed
\begin{figure}[t]
$$
\xymatrix@C=5mm{
&&& (\VV,F):4 \ar[drr]^{\neg p} \ar[dl]_p \\
&& \times: 2 \ar@{..}[dr]\ar@{..}[dl] &&& \times: 2 \ar@{..}[dr] \ar@{..}[dl] \\
& C_1: 2 \ar[d]^q \ar[dl]_{\neg q} && C_2: 1 \ar[d]^{\neg r} \ar[dl]_r & C_6: 1 && C_7:2
\ar[d]^r \ar[dl]_{\neg r} &&& \\
C_3:1 & (\emptyset,\emptyset): 1 & C_4:1 \ar[d]_z  & fail && \times: 1 \ar@{..}[dl] \ar@{..}[d]
& C_{10} = C_4:1  \\
&& C_5: 1 \ar@{-->}[uur] && C_8:1 & C_9:1 &&&&&
}
$$
\caption{A visualization of the search tree for model counting with priority
  variables. Nodes are marked with residual clauses and counts. 
  Dotted edges indicate dynamic decomposition, dashed edged
  indicate backjumps over non-priority decisions.\label{fig:model}}
\end{figure}
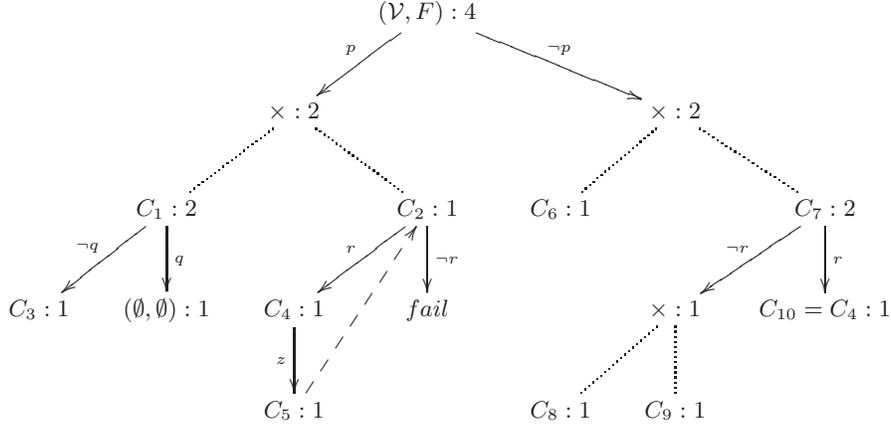
\end{example}

\subsection{Blocking seen solutions}
This approach extends the projected model counting algorithm given in \cite{clasp_projection},
which has been implemented in the ASP solver \clasp \cite{clasp,clasp_journal}.
The algorithm is originally for model enumeration, not model counting, and therefore,
it suffers in instances where there are small number of cubes, but the number of extensions of these
cubes to solutions is large.
We present a modification of the algorithm that does not have this shortcoming. But first, let us
briefly summarize the motivation behind the algorithm and its technical details.

The motivation presented in \cite{clasp_projection} is absence of any specialized
algorithm in SAT (as well as ASP) for model enumeration on a projected set of variables, and the apparent
flaws in the following two straight-forward approaches for model enumeration. The first
is essentially the approach in \ref{subsec:dsharpp} without dynamic decomposition, caching, and cube detection, i.e., to search on variables in $\PP$ first and check for a satisfying extension over $\NN$.
This interference with the search can be exponentially more expensive in the worst case,
although this approach is not compared against other methods in the experiment.
The second approach is to keep track of solutions that have been found and for each
explored solution $\theta$, add the \emph{blocking clause} 
$\neg \theta_{\PP}$ (this is also presented in \cite{sharpcdcl},
although the algorithm restarts and calls \minisat by adding the clause each time a solution is found). In the worst case, the number of solutions
can be exponential in $|\PP|$, and this approach, as experiments confirm, can quickly blow up in space.
Note that, as opposed to the learned clauses which are redundant w.r.t. the original CNF and can
be removed any time during the search, the blocking clauses need to be stored permanently, and cannot
be removed naively.

The algorithm of \cite{clasp_projection} 
runs in polynomial space and works as follows. At any given
time during its execution, the search is divided 
into \emph{controlled} and \emph{free} search. The free 
part of the search runs as an ordinary modern DPLL-based SAT solver would run with backjumping, 
conflict-analysis etc.
In the controlled part of the search, the decision literals are strictly on variables in $\PP$ and how they are chosen is described shortly.
Following the original convention, let $bl$ represent the last level of controlled search space.
Initially, it is equal to $0$. Every time a solution $\theta$ 
(with projection $\theta_P$) is found, 
the search jumps back to $bl$, 
selects a literal $x$ from $\theta_P$ 
that is unfixed (at $bl$), and forces it to be the next decision. It increments $bl$
by 1, adds the blocking clause $\neg \theta_P$ 
and most importantly, \emph{couples} 
the blocking clause with the decision $x$ in the sense that when we backtrack from $x$ and try (force) $\neg x$, $\neg \theta_P$ can be removed from memory
as it is satisfied by $\neg x$. Additionally, backtracking in the controlled region is provably designed to disallow skipping over
any solution. Therefore, when we try $\neg x$, all solutions under $x$ will have been explored. Furthermore, with $\neg x$, all subsequent blocking clauses that were added will have been satisfied since all of them include $\neg x$. This steady removal of clauses ensures that the number of blocking clauses at any given time is in $O(|\PP|)$.


We now describe how we extend the above algorithm by adding solution minimization to it.
We keep a global solution count,
initially set to 0. Once a solution $\theta$ 
is found, we generalize (minimize) the solution as
shown in the procedure \textsf{shrink} Figure~\ref{fig:shrink}. 
We start constructing the new solution cube $S$ by adding all current decisions from 1 \ldots $bl$. Then, for each clause in the 
problem ($C$ in pseudo-code) and current blocking
clauses ($B$), we intersect it with the current assignment. If the intersection contains a literal whose
variable is in $\NN$ or $S$, we skip the clause, otherwise, we add one priority literal from the intersection in $S$
(we choose one with the highest frequency in the original CNF).
After visiting all clauses,
we use $\neg S$ as a blocking clause instead of the one generated by the algorithm above ($\neg \theta_P$).
Finally, we add $2^{|\PP| - |S|}$ to the global count. The rest of the algorithm remains the same.
Note that the decision literals from the controlled part of the search are necessary
to add in the cube, since the algorithm in \cite{clasp_projection} assumes that once
a controlled decision is retracted, all the blocking clauses that were added below it
are satisfied. This could be violated by our solution minimization if
we do not add controlled decisions to $S$.

\begin{figure}[t]
\begin{tabbing}
xx \= xx \= xx \= xx \= xx \= xx \= xx \= \kill
\> \> \> \> let $p \in \theta \cap C \cap \PP$ with highest freq \= \kill\\
\textsf{shrink($\theta$)} \\
\> $S := \{\}$          \>\>\>\> \% universal solution cube         \\
\> \bfor ($i \in 1 \ldots \bl$) \\
\> \> $S.\add(\dec(i))$  \> \> \> \% add decision to $S$ \\
\> \bfor ($c \in C \cup B$) \\
\> \> $f := \false$ \\
\> \> \bfor ($l \in C$) \\
\> \> \> \bif ($l \in \theta$) and ($l \in \NN$ or $l \in S$) \\
\> \> \> \> $f := \true$ \\
\> \> \> \> \bbreak \\
\> \> \bif ($f = \false$) \> \>\> \% if nothing makes the clause true
already \\
\> \> \> let $p \in \theta \cap C : \var(p) \in \PP$ \> \> \% pick $p$ with highest freq.  \\
\> \> \> $S.\add(p)$ \> \> \% add
literal to cube \\
\> $\countt := \countt + 2^{|\PP| - |S|}$ \\
\> $B.\add(\neg S)$ \\
\end{tabbing}
\caption{Pseudo-code for shrinking a solution $\theta$ of original 
  clauses $C$ and blocking clauses $B$ to a solution cube
  $S$, adding its count and a blocking clause to prevent its
  reoccurrence.\label{fig:shrink}}
\end{figure}

\begin{example}
Consider the CNF in Example \ref{ex:main}. Initially, the controlled search part is empty, $B = \emptyset$
and $bl = 0$ as per the original algorithm. Say \clasp finds
the solution: $\{p,\neg q,x,z,r,\neg y\}$. \textsf{shrink} produces the generalized
solution: $S = \{r,p\}$ by parsing
the clauses $(r,z,\neg p)$ and $(p,q)$ 
respectively (all other clauses can be satisfied by \npriority literals).
We increment the model count by 2 ($2^{3 - |S|}$), 
store the blocking clause $\neg S = (\neg r, \neg p)$ and increment
$bl$ by 1. Say, we pick $r$, 
due to the added blocking clause, it propagates $\neg p$, which propagates $q$.
Say that \clasp now finds the solution $\{r,\neg p,q,\neg y,z,x\}$. In \textsf{shrink}, we start by including $r$
in $S$ since that is a forced decision, and then while parsing the clauses, we get $S = \{r,\neg p,q\}$.
Note that if we didn't have to include the blocking clause $(\neg r,\neg p)$, then we could get away
with $S = \{r,q\}$ which would be wrong since that shares the solution 
$\{r,q,p\}$ with the previous cube.
We increment the count to 3 and cannot force any other decision, so we try the decision $\neg r$
in the controlled part. At the same time, upon backtracking, we remove all blocking clauses
from $B$, so it is now empty. Say \clasp finds the solution $\{\neg r, \neg p, q, x, \neg y, z\}$,
\textsf{shrink} gives $S=\{\neg r,\neg p,q\}$. We increment the count to 4, and when we add $\neg S$
as a blocking clause, there are no more solutions under $\neg r$. Therefore,
our final count is 4. The visualization for this example is given in Figure \ref{fig:block}.
\qed
\end{example}

\begin{figure}[t]
$$
\xymatrix@C=5mm{
. \ar@{~>}[d] &    . \ar[d]^r_{(\neg r,\neg p)} &&&
. \ar[dl]^r_{\txt{\scriptsize$(\neg r,\neg p)$\\\scriptsize$(\neg r,p,\neg q)$}}
\ar[dr]_{\neg r} \\
\{p,\neg q,x,r,r,\neg y\}:2 & . \ar@{~>}[d] &&&& . \ar@{~>}[d] \\
& \{r,\neg p,q,\neg y,z,x\}:1 &&&& \{\neg r,\neg p,\neg q,x,\neg y,z\}:1
}
$$
\caption{A visualization of counting models via blocking solutions.
The curly arcs indicate free search, ending in a solution, with an
associated count.
The controlled search is indicated by full arcs, and blocking clauses 
associated with controlled search decisions are shown on arcs.
\label{fig:block}
}
\end{figure}
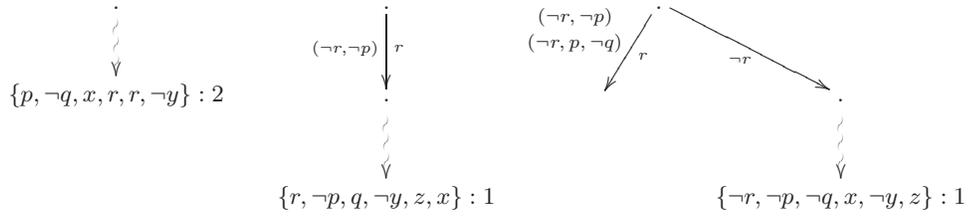

\subsection{Counting models of projected d-DNNF}

As mentioned in Section \ref{sec:preliminaries}, it is possible to do model counting on d-DNNF
in polynomial time (in the size of the d-DNNF), 
however, once we perform projection on $\PP$
(or \emph{forgetting on $\NN$} \cite{kc_map}) by replacing all literals whose variables
are in $\NN$ with $\true$, the resulting logical formula is not deterministic
anymore and model counting is no longer tractable (see \cite{kc_map}).

In this approach, we first compute the d-DNNF of $F$, then project away the literals from the d-DNNF
whose variables are in $\NN$, convert this projected DNNF back to CNF,
and then count the models of this CNF. The pseudo-code is given in Figure
\ref{fig:d2c_pseudo}. The conversion from d-DNNF to CNF is formalized
in the procedure \textsf{d2c}, which takes as its input a d-DNNF (as a list of nodes
$\Nodes$) and returns a CNF $C$. It is assumed that $\Nodes$ is topologically sorted, i.e.,
the children of all nodes appear before their parents. \textsf{d2c} maps nodes to literals in the output CNF with the
dictionary $\litAtNode$. It also maps introduced (Tseitin) variables to expressions that they represent
in a map $\litWithHash$. $v$ represents the index of the next Tseitin variable to be created.
\textsf{d2c} initializes its variables with the method \textsf{init()}. Next, it visits each
node $n$, and checks its type. If it is a literal and if it is a \npriority variable, then it
is replaced with $\true$ (projected away), otherwise, the node is simply mapped to the literal.
If $n$ is an AND or an OR node, then we get corresponding literals of its children from the method \textsf{simplify}.
We compute the hash to see if we can reuse some previous introduced variable instead of introducing
a new one. If not, then we create a new variable through the method \textsf{Tseitin} which also
posts the corresponding equivalence clauses in $C$. Finally, we post a clause that says that the literal for
the root (which is the last node) should be true. The method \textsf{simplify} essentially
maps all the children nodes to their literals. Furthermore, if one of the literals
is $\true$ and the input is an OR-node, it returns a list containing a true literal. 
For an AND node, it filters all the true literals from the children.
 
The next theorem shows that the method described in this section for projected model counting is correct.

\begin{theorem}
$\countt(C) = \countt(F, \PP)$
\end{theorem}
\begin{proof}[sketch]
\newcommand{\DP}{D_{\PP}}
The entire algorithm transforms the theory from $F$ to $C$ by producing 2 auxiliary states:
the d-DNNF of $F$ (let us call it $D$) and the projection of this d-DNNF (let us call this
$\DP$). By definition, $F$ and $D$ are logically equivalent. 
On the other end, notice that the models of $\DP$ and $C$ are in one-to-one correspondence.
Although the two are not logically equivalent due to the addition of Tseitin variables, 
it can be shown that these variables do not introduce any extra model nor eliminate any existing model 
since they are simply functional definitions of variables in $\PP$ by construction
(as a side note, the only reason for introducing these variables is to efficiently encode $\DP$ as CNF, 
otherwise, $C$ and $\DP$ would be logically equivalent).
Furthermore, we can show that the simplifications (replacing $\true \vee E$ with
$\true$ and $\true \wedge E$ with $E$) in the procedure \textsf{simplify},
and reusing Tseitin variables (through hashing) also do not affect the bijection.
This just leaves us with the task of establishing bijection between the models of $D$ and $\DP$,
which, fortunately, has already been done in \cite{ddnnf}. Theorem 9 in the paper says
that replacing \npriority literals with true literals in a d-DNNF is a proper projection
operation, and Lemma 3 establishes logical equivalence between $D$ and $\DP$ modulo variables in $\PP$.
\end{proof}

\begin{figure}[t]
\scriptsize
\raggedright
\begin{minipage}{0.4\textwidth}
\begin{tabbing}
xx \= xx \= xx \= xx \= xx \= xx \= xx \= \kill
\textsf{d2c($\Nodes$)} \\
\> \textsf{init()} \\
\> \bfor ($n \in \Nodes$) \\
\> \> \bif($n$ is a literal $l$) \\
\> \> \> \bif($\var(l) \in \NN$) \\
\> \> \> \> $\litAtNode[n] := \true$ \\
\> \> \> \belse $\litAtNode[n] := l$ \\
\> \> \belseif($n = \op(c_1, \ldots, c_k)$) \\
\> \> \> $(l_1, \ldots, l_j) :=$ \textsf{simplify($n$)}\\
\> \> \> \bif ($j = 1$) \\
\> \> \> \> $\litAtNode[n] := l_1$ \\
\> \> \> \belse \\
\> \> \> \> $h :=$ \textsf{hash$(\op, (l_1, \ldots, l_j))$} \\
\> \> \> \> \bif ($\litWithHash.\textsf{hasKey}(h)$) \\
\> \> \> \> \> $\litAtNode[n] := \litWithHash[h]$ \\
\> \> \> \> \belse \\
\> \> \> \> \> $v' := $ \textsf{Tseitin$(\op, (l_1, \ldots, l_j))$} \\
\> \> \> \> \> $\litAtNode[n] := v'$ \\
\> \> \> \> \> $\litWithHash[h] := v'$ \\
\> $C.\add(\{\litAtNode[\Nodes.\textsf{last()}]\})$ \\
\> \breturn $C$
\end{tabbing}
\end{minipage}\qquad
\begin{minipage}{0.4\textwidth}
\begin{tabbing}
xx \= xx \= xx \= xx \= xx \= xx \= xx \= \kill
\textsf{init()} \\
\> $C = ()$, $\litAtNode = \{\}$, $\litWithHash = \{\}$ \\
\> $v := |\VV|$
\\ \\
\textsf{simplify($\op(c_1, \ldots, c_k)$)} \\
\> $L = ()$ \\
\> \bfor ($c \in c_1, \ldots, c_k$) \\
\> \> \bif ($\litAtNode[c] = \true$) \\
\> \> \> \bif ($op =$ OR) \breturn $(\true)$ \\
\> \> \belse \\
\> \> \> $L.\add(\litAtNode[c])$ \\
\> \breturn $L$
\\ \\
\textsf{Tseitin$(\op, (l_1, \ldots, l_j))$} \\
\> Add clauses $v \Leftrightarrow \op(l_1, \ldots, l_j)$ in $C$ \\
\> $v := v + 1$ \\
\> \breturn $v - 1$
\end{tabbing}
\end{minipage}
\caption{Pseudo-code for projected model counting
via counting models of CNF encoding of projected DNNF.\label{fig:d2c_pseudo}}
\end{figure}

\newcommand*{\Scale}[2][4]{\scalebox{#1}{$#2$}}%

\begin{example}
Consider the formula $F$ with \priority variables $p,q$ and \npriority variables $x,y,z$:
$$
(\neg x,p), (q, \neg x, y), (\neg p,\neg y,\neg z,q), (x,q), (\neg q,p)
$$
The projected model count is 2 ($(p,q)$ and $(p,\neg q)$).
\begin{figure}[h]
        \centering
        \begin{subfigure}{0.3\textwidth}
                $
								\Scale[0.6]{
								\xymatrix{
&& \vee \ar[dl]\ar[dr] \\
									&
                                                                        \wedge
                                                                        \ar[dl]
                                                                        \ar[d]
                                                                        \ar[dr]
                                                                        & &
                                                                        \wedge
                                                                        \ar[d]
                                                                        \ar[ddl]
                                                                        \ar[dr]
                                                                        \\
									x &
                                                                        p &
                                                                        \vee
                                                                        \ar[dl]
                                                                        \ar[d]
                                                                        &
                                                                        \neg
                                                                        x & p \\
									& \wedge \ar[dl] \ar[d] \ar[dr] & q & \\
									\neg q & y & \neg z \\
								}
								}
								$
                \caption{d-DNNF}
                \label{fig:ddnnf}
        \end{subfigure}%
        \begin{subfigure}{0.3\textwidth}
                $
								\Scale[0.6]{
								\xymatrix{
& \stackrel{a_4}{\vee} \ar[dl]\ar[dr] \\
\stackrel{a_2}{\wedge}  \ar[d]   \ar[dr]
                 &  &
\stackrel{a_3}{\wedge} \ar[d] \ar[ddl] \\
		p &
\stackrel{a_1}{\vee} \ar[dl] \ar[d] & p \\
									\neg q & q & \\
								}
								}
								$
								\caption{Projected d-DNNF}
                \label{fig:projected_ddnnf}
        \end{subfigure}
				\begin{subfigure}{0.3\textwidth}
								$$
								\begin{array}{l}
								a_1 \Leftrightarrow q \vee \neg q \\
								a_2 \Leftrightarrow a_1 \wedge p \\
								a_3 \Leftrightarrow p \wedge q \\
								a_4 \Leftrightarrow a_2 \vee a_3 \\
								a_4
								\end{array}
								$$
                \caption{Formula from d2c}
                \label{fig:tseitin}
        \end{subfigure}
				
        \caption{Example of application of d2c}
				\label{fig:d2c}
\end{figure}
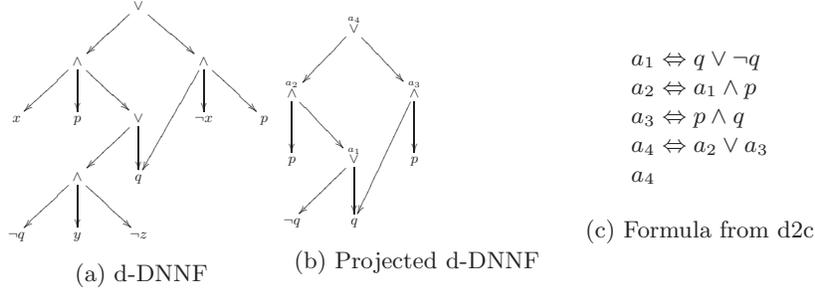

Figure \ref{fig:d2c} shows the initial d-DNNF (\ref{fig:ddnnf}), the DNNF obtained
by replacing all \npriority literals by $\true$ and simplifying (\ref{fig:projected_ddnnf})
and the d2c translation of the projected DNNF (\ref{fig:tseitin}).
Notice that if we perform model counting naively on the projected DNNF, we get a count
of 3 since we double count the model $(p,q)$. 
The satisfaction probability is: 
$$ 
(\frac{1}{2} + \frac{1}{2}) \times \frac{1}{2} + (\frac{1}{2} \times \frac{1}{2}) = \frac{3}{4}
$$
From satisfaction probability, we get the wrong model count $2^2 \times \frac{3}{4} = 3$.
However, if we count the models of the translated formula in
(\ref{fig:tseitin}), we get the correct count of 2.
\qed
\end{example}


\section{Experiments}

We compare the following solvers on various benchmarks: \clasp in its projection mode (\clasp),
our extension of clasp with cube minimization (\sharpclasp),
model counting with searching on priority variables first (\dsharpp), and counting models
of projected DNNF (\dtoc). 
%
%
In each row of the following tables, $|\PP|$ is the number of priority variables.
$T$ and $D$ represent the execution time and number of decisions taken by the solver.
$R$ is a parameter to gauge the quality of cubes computed by \sharpclasp, the higher it
is, the better. It is equal to $\solCubeRatio$. A value of 0 indicates that all solution
cubes computed have size 1, while the maximum value is equal to the number of \priority variables,
which is the unique case when there is only one cube and every assignments to \priority variables
is a solution. $R$ essentially quantifies the advantage over enumeration, the less constrained
a problem is, and the more general the cubes are, the higher the advantage.
%
%
$S$ is the size (in bytes) of the CNF computed by \dtoc
that is subsequently given to the solver \sharpsat for model counting.
The timeout for all experiments is 10 minutes. All times are shown in seconds.
The experiments were run on NICTA's HPC cluster.
\footnote{All benchmarks and solvers are available at:
\url{http://people.eng.unimelb.edu.au/pstuckey/countexists}}

\subsection{Uniform random 3-SAT and Boolean circuits}

Table \ref{table:uf_circuit} shows the results from uniform random 3-SAT and random Boolean circuits. 
In this table, for each problem instance, we show how the solvers perform as we increase the number
of priority variables. A ``\ldots'' after a row means that every solver either ran out of time or memory
for all subsequent number of priority variables until the next one shown.
For each instance, a row is added that provides the following information about it:
name, number of solutions as reported by \dsharp, number of variables and clauses and time and decisions taken by \dsharp. 
Note that this time should be added to the time of \dtoc in order to get the actual time of \dtoc approach.

Let us look at the results form uniform random 3-SAT. All instances have 100 variables,
and the number of clauses is varied. We try clause-to-variable ratios of 1, 1.5, 2, 3 and 4.
Note that for model counting, the difficulty peaks at the ratio of approximately 1.5 \cite{model_counting}.
For the first 3 instances, \sharpclasp is the clear winner while \clasp also does well, \dsharpp
lags behind both, and \dtoc does not even work since the original instance cannot be solved by \dsharp.
For \sharpclasp, as we increase the number of clauses, the cube quality decreases due to
the problem becoming more constrained and cube minimization becoming less effective.
For 300 clauses, we see a significant factor coming into play for \dsharpp. The original instance is solved
by \dsharp. As we increase the number of priority variables until nearly the middle, the performance
of \dsharpp degrades, but after 50 \priority variables, it starts getting better. This is because the degradation
due to searching on \priority variables first becomes less significant and the search
starts working more naturally in its VSADS mode \cite{vsads}. \dtoc also solves two rows
in this instance but is still largely crippled as compared to other solvers.
Finally, with 400 clauses, we are well past the peak difficulty and the number of models is small enough
to be enumerated efficiently by \clasp. All solvers finish all rows of this instance in less than .15 seconds.
We tried the same ratios for 200 and 300 variables. For 200 variables, we saw the same trend, although
the problem overall becomes harder and the number of solved rows decreases. For 300 variables, the problem
becomes significantly harder to be considered a suitable benchmark.

The Boolean circuits are generated with $n$ variables as follows: we keep a set initialized
with the $n$ original variables, then as long as the set is not a singleton, we randomly pick an operator $o$
(AND, OR, NOT), remove random operands $V$ from the set, create a new variable $v$ and post the constraints
$v \leftrightarrow o(V)$ and put $v$ back in the set. The process is repeated $c$ times.
In the table, we show the results where $n$ is 30, and $c$ is 1,5,10. Note that a higher value
of $c$ means that the problem is more constrained. Overall, for all instances, \dsharpp is the superior
approach, followed by \clasp; and \dtoc is better than \sharpclasp in $c=1$ but the converse
is true for higher values of $c$. All solvers find $c=5$ to be the most difficult instance.
We saw similar trends for different values of $n$ that have appropriate hardness with same values of $c$.


\begin{center}
\scriptsize
\begin{longtable}{rr|rr|rrr|rr|rrr}
\hline
 & & \multicolumn{2}{c|}{\clasp} & \multicolumn{3}{c|}{\sharpclasp} &  \multicolumn{2}{c|}{\dsharpp} & \multicolumn{3}{c}{\dtoc} \\ [0.2ex]
$|\PP|$ & $\#$ & $T$ & $D$ & $T$ & $D$ & $R$ & $T$ & $D$ & $T$ & $D$ & $S$ \\ [0.2ex]

\hline \\
& & \multicolumn{10}{c}{UF $\#$=--- $|V|$=100 $|C|$=100 $T$=--- $D$=---} \\ 
 \hline
5 & 32 & 0 & 2271 & 0 & 291 & 3.00 & .04 & 1150 & --- & --- & --- \\[0.2ex]
10 & 1024 & .01 & 71309 & 0 & 533 & 7.19 & .99 & 35927 & --- & --- & --- \\[0.2ex]
15 & 32768 & .40 & 2023146 & 0 & 1888 & 10.30 & 7.92 & 370034 & --- & --- & --- \\[0.2ex]
25 & 2.7e+07 & 345.57 & 1.4e+09 & 0 & 10584 & 17.36 & --- & --- & --- & --- & --- \\[0.2ex]
35 & 1.8e+10 & --- & --- & .02 & 62016 & 24.32 & --- & --- & --- & --- & --- \\[0.2ex]
50 & 1.9e+14 & --- & --- & 107.75 & 7.1e+07 & 27.04 & --- & --- & --- & --- & --- \\[0.2ex]
$\ldots$ \\[0.2ex]
\hline \\
& & \multicolumn{10}{c}{UF $\#$=--- $|V|$=100 $|C|$=150 $T$=--- $D$=---} \\ 
 \hline
5 & 32 & 0 & 1937 & 0 & 247 & 3.00 & .03 & 1286 & --- & --- & --- \\[0.2ex]
10 & 1024 & .02 & 54077 & .01 & 2933 & 4.19 & .68 & 27142 & --- & --- & --- \\[0.2ex]
15 & 32768 & .42 & 1767073 & 0 & 2101 & 9.96 & 31.63 & 1057551 & --- & --- & --- \\[0.2ex]
25 & 2.1e+07 & 270.91 & 8.7e+08 & .34 & 393130 & 11.25 & --- & --- & --- & --- & --- \\[0.2ex]
35 & 2.8e+09 & --- & --- & 24.98 & 1.5e+07 & 12.84 & --- & --- & --- & --- & --- \\[0.2ex]
$\ldots$ \\[0.2ex]
\hline \\
& & \multicolumn{10}{c}{UF $\#$=--- $|V|$=100 $|C|$=200 $T$=--- $D$=---} \\ 
 \hline
5 & 32 & 0 & 1354 & 0 & 259 & 2.42 & .04 & 1304 & --- & --- & --- \\[0.2ex]
10 & 1024 & .01 & 47771 & 0 & 1370 & 5.25 & 1.12 & 37596 & --- & --- & --- \\[0.2ex]
15 & 30712 & .43 & 1408296 & 0 & 4659 & 8.29 & 37.07 & 874826 & --- & --- & --- \\[0.2ex]
25 & 1.8e+07 & 218.20 & 6.4e+08 & 1.69 & 1801261 & 8.54 & --- & --- & --- & --- & --- \\[0.2ex]
$\ldots$ \\[0.2ex]
\hline \\
& & \multicolumn{10}{c}{UF $\#$=2.603e+11 $|V|$=100 $|C|$=300 $T$=31.44 $D$=571163} \\ 
 \hline
5 & 32 & 0 & 986 & 0 & 646 & 0.75 & .05 & 671 & 40.92 & 31 & 865K \\[0.2ex]
10 & 970 & .02 & 25441 & .02 & 11450 & 1.20 & 2.07 & 14146 & 102.63 & 969 & 1.6M \\[0.2ex]
15 & 12990 & .22 & 290973 & .11 & 61663 & 2.30 & 12.74 & 144211 & --- & --- & 6.7M \\[0.2ex]
25 & 226117 & 3.84 & 3432170 & 1.66 & 464908 & 3.21 & 57.57 & 808670 & --- & --- & 15M \\[0.2ex]
35 & 5126190 & 49.02 & 6.6e+07 & 15.65 & 3367386 & 4.67 & 161.67 & 2834211 & --- & --- & 38M \\[0.2ex]
50 & --- & --- & --- & --- & --- & --- & --- & --- & --- & --- & 61M \\[0.2ex]
65 & 1.6e+09 & --- & --- & --- & --- & --- & 70.80 & 1552565 & --- & --- & 89M \\[0.2ex]
75 & 2.0e+10 & --- & --- & --- & --- & --- & 70.74 & 1330586 & --- & --- & 104M \\[0.2ex]
85 & 2.9e+10 & --- & --- & --- & --- & --- & 50.18 & 780597 & --- & --- & 113M \\[0.2ex]
100 & 2.6e+11 & --- & --- & --- & --- & --- & 28.62 & 571163 & --- & --- & 134M \\[0.2ex]
\hline \\
& & \multicolumn{10}{c}{UF $\#$=45868 $|V|$=100 $|C|$=400 $T$=.05 $D$=244} \\ 
 \hline
5 & 7 & 0 & 1078 & 0 & 907 & 0.49 & .08 & 219 & .01 & 6 & 549 \\[0.2ex]
10 & 25 & 0 & 1308 & 0 & 1103 & 0.94 & .14 & 322 & .01 & 17 & 1.3K \\[0.2ex]
15 & 32 & 0 & 1582 & 0 & 1242 & 1.09 & .12 & 376 & .01 & 21 & 3.1K \\[0.2ex]
25 & 105 & .01 & 2242 & 0 & 1290 & 2.32 & .09 & 373 & .01 & 34 & 3.4K \\[0.2ex]
35 & 246 & .01 & 3068 & 0 & 1338 & 2.52 & .06 & 363 & .01 & 107 & 8.6K \\[0.2ex]
50 & 952 & .01 & 6737 & .01 & 2241 & 3.24 & .05 & 361 & .02 & 202 & 14K \\[0.2ex]
65 & 3417 & .01 & 16388 & .01 & 2889 & 4.41 & .05 & 262 & .09 & 321 & 21K \\[0.2ex]
75 & 7964 & .04 & 26979 & .02 & 2845 & 5.32 & .05 & 250 & .04 & 426 & 22K \\[0.2ex]
85 & 13274 & .04 & 36445 & .02 & 3993 & 5.18 & .05 & 237 & .06 & 563 & 26K \\[0.2ex]
100 & 45868 & .11 & 46623 & .03 & 4639 & 6.74 & .04 & 244 & .07 & 688 & 31K \\[0.2ex]
\hline \\
& & \multicolumn{10}{c}{$n$=30 $c$=1 $\#$=9.657e+08 $|V|$=99 $|C|$=167 $T$=0 $D$=111} \\ 
 \hline
5 & 16 & 0 & 409 & 0 & 292 & 0.54 & 0 & 113 & .01 & 4 & 1.1K \\[0.2ex]
9 & 160 & 0 & 3418 & 0 & 1184 & 1.54 & 0 & 143 & 0 & 8 & 1.2K \\[0.2ex]
14 & 552 & 0 & 9305 & 0 & 5030 & 1.00 & 0 & 82 & .01 & 22 & 3.8K \\[0.2ex]
24 & 248960 & 1.16 & 2718019 & 1.49 & 833682 & 2.07 & 0 & 130 & .01 & 169 & 6.1K \\[0.2ex]
34 & 1621760 & 6.13 & 1.2e+07 & 6.45 & 1999088 & 3.13 & .01 & 111 & .05 & 656 & 9.7K \\[0.2ex]
49 & 3.9e+07 & 104.26 & 1.9e+08 & 353.25 & 1.4e+07 & 4.78 & .01 & 162 & .10 & 1393 & 14K \\[0.2ex]
64 & 1.5e+08 & 394.21 & 4.6e+08 & --- & --- & --- & 0 & 129 & .18 & 2143 & 18K \\[0.2ex]
74 & 4.4e+08 & --- & --- & --- & --- & --- & .01 & 108 & .21 & 2982 & 20K \\[0.2ex]
84 & 7.2e+08 & --- & --- & --- & --- & --- & 0 & 99 & .20 & 2624 & 24K \\[0.2ex]
99 & 9.7e+08 & --- & --- & --- & --- & --- & .01 & 111 & .20 & 2517 & 27K \\[0.2ex]
\hline \\
& & \multicolumn{10}{c}{$n$=30 $c$=5 $\#$=9.426e+07 $|V|$=389 $|C|$=867 $T$=288.45 $D$=1036363} \\ 
 \hline
19 & 12192 & .16 & 155331 & .75 & 146058 & 0.00 & 16.48 & 120619 & --- & --- & 5.8M \\[0.2ex]
38 & 208716 & 2.57 & 1882991 & 15.23 & 1985136 & 0.00 & 95.37 & 834705 & --- & --- & 47M \\[0.2ex]
58 & 1.2e+07 & 93.69 & 3.7e+07 & --- & --- & --- & --- & --- & --- & --- & 100M \\[0.2ex]
97 & 3.3e+07 & 248.76 & 6.9e+07 & --- & --- & --- & 427.85 & 1509308 & --- & --- & 171M \\[0.2ex]
136 & 6.1e+07 & 428.89 & 9.1e+07 & --- & --- & --- & --- & --- & --- & --- & 252M \\[0.2ex]
... &&&&&&&&&&& \\[0.2ex]
291 & 9.3e+07 & --- & --- & --- & --- & --- & 300.78 & 985065 & --- & --- & 574M \\[0.2ex]
330 & 9.4e+07 & --- & --- & --- & --- & --- & 299.02 & 1074927 & --- & --- & 672M \\[0.2ex]
389 & 9.4e+07 & --- & --- & --- & --- & --- & 308.84 & 1036363 & --- & --- & 783M \\[0.2ex]
\hline \\
& & \multicolumn{10}{c}{$n$=30 $c$=10 $\#$=5066 $|V|$=766 $|C|$=1771 $T$=.32 $D$=1400} \\ 
 \hline
38 & 282 & .01 & 1196 & .03 & 1797 & 0.00 & .31 & 1412 & .08 & 256 & 36K \\[0.2ex]
76 & 1618 & .02 & 3479 & .12 & 5434 & 0.00 & .40 & 1600 & .81 & 2046 & 137K \\[0.2ex]
114 & 2581 & .03 & 4984 & .21 & 7953 & 0.00 & .52 & 1787 & 1.69 & 3702 & 173K \\[0.2ex]
191 & 4948 & .05 & 5558 & .52 & 12243 & 0.00 & .54 & 1904 & 4.98 & 7519 & 330K \\[0.2ex]
268 & 5066 & .07 & 5458 & .63 & 12235 & 0.00 & .38 & 1784 & 6.69 & 10508 & 478K \\[0.2ex]
383 & 5066 & .09 & 5513 & .85 & 12356 & 0.00 & .69 & 1975 & 9.27 & 12528 & 698K \\[0.2ex]
497 & 5066 & .08 & 5253 & 1.47 & 12471 & 0.00 & .39 & 1680 & 12.46 & 11818 & 1.1M \\[0.2ex]
574 & 5066 & .11 & 5211 & 1.68 & 12358 & 0.00 & .52 & 1616 & 14.85 & 11911 & 1.1M \\[0.2ex]
651 & 5066 & .09 & 5500 & 1.21 & 12546 & 0.00 & .36 & 1500 & 13.42 & 11807 & 1.4M \\[0.2ex]
766 & 5066 & .09 & 5072 & 2.06 & 12389 & 0.00 & .33 & 1400 & 21.61 & 11644 & 1.6M \\[0.2ex]
\hline \\
\caption{Results from random uniform 3-SAT and Boolean circuits}
\label{table:uf_circuit}
\end{longtable}
\end{center}

\subsection{Planning}

\newcommand{\finished}{\checkmark}

Table \ref{table:planning} summarizes performance of different projected model counting
algorithms on checking robustness of partially ordered plans to initial conditions.
We take five planning benchmarks: depots, driver, rovers, logistics, and storage. For each
benchmark, we have two variants, 
one with the goal state fixed and one where the goal is relaxed to be any viable goal
(shown with a capital A in the table representing \emph{any} goal).
For the two variants, the \priority variables are defined such that by doing
projected model counting, we count the following.
For the first problem, we count the number of initial states the given plan can achieve the given goal from.
For the second problem, we count the number of initial states plus all goal configurations that the given plan works for.
Each row in the table represents the summary of 10 instances of same size. The first 3 columns
show the instance parameters.
For each solver, \finished shows how many instance the solver was able to finish within
time and memory limits. All other solver parameters are averages over finished instances.
Another difference from the previous table is that we have added the execution time
of \dsharp in \dtoc and \dsharp time is shown in parenthesis. 
There was no case in which only \dsharp finished and the remaining steps of \dtoc
did not finish. 

Overall, \dsharpp solves the most instances (42), followed by \sharpclasp (41), 
\dtoc (34), and finally \clasp which solves only 4 instances from the storage benchmark,
and otherwise suffers due to the inability to detect cubes.
\dsharpp and \dtoc only fail on all instances in 2 benchmarks while \sharpclasp
fails in 4, so they are more robust in that sense.
For \dtoc, the running time is largely taken by producing the d-DNNF and
the second round of model counting is relatively cheaper.
The cube quality of \sharpclasp is quite significant for all instances
that it solves.

\begin{center}
\scriptsize
\begin{longtable}{rrrr|lrr|lrrr|lrr|lrrr}
\hline
\multicolumn{4}{c|}{Instance} & \multicolumn{3}{c|}{\clasp} & \multicolumn{4}{c|}{\sharpclasp} &  \multicolumn{3}{c|}{\dsharpp} & \multicolumn{4}{c}{\dtoc} \\ [0.2ex]

Name & $|\VV|$ & $|C|$ & $|\PP|$ &
\finished & $T$ & $D$ & 
\finished & $T$ & $D$ & $R$ & 
\finished & $T$ & $D$ & 
\finished & $T$ & $D$ & $S$ \\ [0.2ex] 
\hline \\

depotsA            & 9402 & 211901 & 224 & 0 & --- & --- & 0 & --- & --- & --- & 1 & 24.82 & 92206 & 1 & 8.34 (6.98) & 1813 & 154K \\ [0.2ex]
depots    & 9211 & 211796 & 111.8 & 0 &	--- &	--- & 2 & 4.16 & 4.43e+6 &	31.54 & 1 & 24.74 &	91909 & 1 &	7.71 (6.67) & 1642 & 149K \\[0.2ex]
driverA            & 2068	& 12798 & 135 & 0 & --- &	--- &	0 &	--- &	--- &	--- & 5 & 36.01 & 27104.8	& 3 & 164.73 (161.42) & 68.33 & 150.5K \\ [0.2ex]
driver    & 1999 & 12700 & 68 &	0 & -- &	--- & 10 & 0.31 & 1.23e+5 & 51.7 &	5 &	15.8 & 1.45e+4 & 3 & 118.67 (116.00) & 29.3 & 109K \\ [0.2ex]
logisticsA         & 18972 & 324568 & 447 & 0 & --- & --- & 0 & --- & --- & --- & 0 & --- & --- & 0 & --- & --- & --- \\ [0.2ex]
logistics & 18702 & 324352 & 224 & 0 & --- & --- & 6 & 33.52 & 1.81e6 & 165.09 & 0 & --- & --- & 0 & --- & --- & --- \\ [0.2ex]
roversA            & 3988 & 27634 &	209 & 0 & ---	& --- &	0 & --- & --- &	--- & 5 & 69.92 & 51965 & 3 & 1.11 (1.06) & 53.33 & 5.37K \\ [0.2ex]
rovers    & 3851 & 27535 & 104 & 0 &	--- &	--- & 10 & 0.30 & 36769.7 &	88.16 & 5 &	76.26 &	52245.4 & 3 & 1.04 (1.01) & 12.67 & 3.4K \\ [0.2ex]
storageA & 915 & 3465 & 93 & 1& 454.2 & 2.5e9 & 3 & 43.81 & 3.89e7 & 18.01 & 10 & 49.04 &	47112.50 & 9 & 103.35 (78.60) & 1964.2 & 440.21K \\ [0.2ex]
storage    & 851 & 3420 & 47 & 3 & 15.05 & 7.87e7 &	10 & 0.05 & 30686 & 30.47 &	10 & 15.48 & 12444 & 9 & 57.1 (53.46) & 625.67 & 254.58K \\ [0.2ex]
\caption{Results from robustness of partially ordered plans to initial conditions.}
\label{table:planning}
\end{longtable}
\end{center}

\section{Related Work and Conclusion}


The area of Boolean Quantifier Elimination (BQE) seems closely related to projected model
counting. Although the goal in BQE is to produce a \npriority variables free representation (usually CNF),
some algorithms can also be adapted for projected model counting. 
Of particular interest are techniques in \cite{existential_quantification} and \cite{dds}.
The algorithm of \cite{existential_quantification} finds cubes in decreasing (increasing) order of cube size
(number of literals in cube).
While this approach does not require
cube minimization, it does not run in polynomial space, and 
if a problem only has large cubes, then significant time might be wasted
searching for smaller ones. 
The second interesting approach, with promising results, is given in \cite{dds}, which uses a DPLL-style
search and also decomposes the program at each step like the approach described in \ref{subsec:dsharpp}.
A good direction for future work is to investigate how well these techniques
lend themselves to projected model counting and whether there is any room
for integration with the ideas presented in this paper.

In this paper we compare four algorithms for projected model counting. 
We see that each algorithm can be superior in appropriate
circumstances: 
\begin{itemize}
\item When the number of solutions is small then 
\clasp~\cite{clasp_projection} is usually the best.
\item When the number of solution cubes is much smaller than
solutions, and there is not much scope for component caching, then
\sharpclasp is the best.
\item When component caching and dynamic decomposition are useful then
\dsharpp is the best.
\item Although \dtoc is competitive, it rarely outperforms both \sharpclasp and \dsharpp.
	Having said that, \dtoc approach has another important aspect besides
	projected model counting. It is a method to perform projection
	on a d-DNNF without losing determinism. This can be done by
	computing the d-DNNF of the CNF produced
	by the \textsf{d2c} procedure (instead of model counting), and then simply forgetting
	the Tseitin variables (replacing with $\true$). It can be shown that
	this operation preserves determinism. Furthermore, our experiments show
	that the last model counting step takes comparable time to computing
	the first d-DNNF in most cases (and in many cases, takes significantly less time),
	which means that the approach is an efficient way of performing
	projection on a d-DNNF.	
	
	While we use d-DNNF for \dtoc approach, it is possible to use other, less succinct, 
	languages like Ordered Binary Decision Diagrams (OBDDs). 
	We leave the comparison with other possible knowledge compilation-based approaches for 
	projected model counting as future work.
\end{itemize}
As the problem of projected model counting is not heavily explored,
there is significant scope for improving 
algorithms for it.
A simple improvement would be to portfolio approach to solving the problem,
combining all four of the algorithms, to get something close
to the best of each of them.

\bibliographystyle{splncs03}
\bibliography{paper}

\end{document}